\RequirePackage{amsmath}
\RequirePackage{amsthm}
\documentclass{llncs}
\usepackage{mathrsfs}
\usepackage{marvosym}
\usepackage{amssymb}
\usepackage{amsbsy}
\usepackage[T1]{fontenc}
\usepackage{manfnt}
\usepackage{algorithm}
\usepackage[noend]{algpseudocode}
\usepackage{caption}
\usepackage{datetime}
\usepackage{listings}
\usepackage{mathtools}
\usepackage{stmaryrd}
\usepackage{nicefrac}
\usepackage[autostyle]{csquotes}
\usepackage{listings}
\usepackage[hyphens]{url}
\usepackage{microtype}

\makeatletter
\def\BState{\State\hskip-\ALG@thistlm}
\makeatother

\theoremstyle{plain}

\newtheorem{notation}{Notation}

\newcommand{\ZPA}{\( \mathbf{0}' \) algorithm}
\newcommand{\ZPAs}{\( \mathbf{0}' \) algorithms}

\DeclareCaptionFormat{algor}{%
  \hrulefill\par\offinterlineskip\vskip1pt%
    \textbf{#1#2}#3\offinterlineskip\hrulefill}
\DeclareCaptionStyle{algori}{singlelinecheck=off,format=algor,labelsep=space}
\let\oldsqrt\sqrt
\def\sqrt{\mathpalette\DHLhksqrt}
\def\DHLhksqrt#1#2{%
\setbox0=\hbox{$#1\oldsqrt{#2\,}$}\dimen0=\ht0
\advance\dimen0-0.2\ht0
\setbox2=\hbox{\vrule height\ht0 depth -\dimen0}%
{\box0\lower0.4pt\box2}}

\def\zp{\mathbf{0}'}

\algdef{SN}[dop]{StartDo}{EndDo}{do }{}

\begin{document}

\title{Analysis of Algorithms and Partial Algorithms}
\author{Andrew MacFie}
\institute{\ }
%\newdateformat{amdate}{\THEYEAR.\shortmonthname[\THEMONTH].\twodigit{\THEDAY}}
%\date{\vspace{-5ex}}
%\date{\amdate\today}
\maketitle

\begin{abstract}
  We present an alternative methodology for the analysis of algorithms,
  based on the concept of expected discounted reward.
  This methodology naturally handles algorithms that do not always terminate,
  so it can (theoretically) be used with partial
  algorithms for undecidable problems, such as those found in artificial
  general intelligence (AGI) and automated theorem proving.
  We mention an approach to self-improving AGI enabled by this methodology.
\end{abstract}

%\tableofcontents  % JFT

\section{Introduction: Shortcomings of Traditional Analysis of Algorithms}

Currently, the (running time) analysis of algorithms takes the following form.
Given two algorithms \( A \), \( B \) that solve the same problem,
we find which is more efficient by asymptotically comparing the running time
sequences \( (a_n) \), \( (b_n) \) \cite{clr,iaofa}.
This could be using worst-case or average-case running times or even
smoothed analysis \cite{spielman2009smoothed}.
We refer to this general method as \emph{traditional analysis of algorithms}.

As with any model, traditional analysis of algorithms is not perfect.
Authors have noted
\cite{aaronson2012philosophers,gurevich1993feasible}
that comparing sequence tails avoids the arbitrariness of any particular range
of input lengths but leads us to say \( a_n = n^{100} \) is superior to
\( b_n = \left( 1 + \exp(-10^{10}) \right )^n \) which is false for
practical purposes.

A further issue with traditional analysis is illustrated by
this situation:
Say we have a function \(F: \{0, 1\}^* \rightarrow \{0, 1\} \)
and an algorithm \( A \) that computes \(F\) such that for
\( n \geq 0 \), \(A\) takes \( (n!)! \) steps on the input \( 0^n \)
and \( n \) steps on any other input of length \(n\).
The algorithm \(A\) then has worst-case running time \( (n!)! \) and
average-case running time slightly greater than \( 2^{-n} (n!)! \),
which are both terrible.
However, if the inputs are generated according to a uniform distribution, the
probability of taking more than \( n \) steps is  \( 2^{-n} \) which is quickly
negligible.
We see that \(A\) should be considered an excellent algorithm but traditional
analysis does not tell us that, unless we add ``with high probability''.

The same issue arises if \( A \) simply does not halt on \(0^n\),
in which case the worst-case and average-case running times are infinite.
Indeed, this is not an esoteric phenomenon.
For any problem with Turing degree \( \zp \) we cannot have an algorithm
that halts on every input, but we develop
partial solutions that work on a subset of inputs.
Such problems include string compression (Kolmogorov complexity),
the halting problem in program analysis \cite{looper}, algebraic
simplification \cite{trott}, program optimization, automated theorem proving,
and Solomonoff induction (central to artificial general intelligence
\cite{lv}).
E.g.\ in the case of automated theorem proving, Buss, describing the main
open problems in proof theory \cite{2000}, states,
``Computerized proof search ... is widely used, but almost no mathematical
theory is known about the effectiveness or optimality of present-day
algorithms.''

\begin{definition}
  An algorithm $A$ is a \emph{partial algorithm} (a.k.a.\
  computational method \cite[p5]{knut}) for a given problem if on all
  inputs, $A$ either outputs the correct value, or does not terminate.
\end{definition}

\begin{definition}
  We refer to partial algorithms for problems with Turing degree \( \zp \)
  as \emph{\( \zp \) algorithms}.
\end{definition}

To analyze \ZPAs, and perhaps to better analyze normal terminating algorithms,
we need a new approach that is not based on worst-case or average-case running
time sequences.
In Sect.~\ref{sec:er} we present a new method for analyzing algorithms,
called expected-reward analysis that avoids some of the issues mentioned above.
Then in Sect.~\ref{sec:selfimprovement} we mention how this method can be
used in self-improving AI systems.
We give directions for further work in Sect.~\ref{sec:furtherwork}.

\begin{notation}
  Given a (possibly partial) algorithm \(A\) and an input \(\omega\),
  we denote the number of steps taken by \(A\) on \(\omega\) by
  \( c_A(\omega),\)
  which takes the value \( \infty \) if \(A\) does not halt on \(\omega\).
\end{notation}

\section{Expected-Reward Analysis of Algorithms} \label{sec:er}

\subsection{Definition}

Let $A$ be a (possibly partial) algorithm with inputs in \( \Omega \).
We say the \emph{score} of $A$ is
\begin{equation*}
  S(A)
    = \sum_{\omega \in \Omega} P(\{\omega\}) r(\omega) D(c_A(\omega))
    = E(r \cdot (D \circ c_A)) \enspace,
\end{equation*}
where
\( P \) is a probability measure on \( \Omega \),
$D$ is a discount function \cite{frederick2002time},
and
$r(\omega)$ is a reward (a.k.a.\ utility) value associated with obtaining the
solution to $\omega$.
The expression $S(A)$ may be interpreted as the expected discounted reward that
$A$ receives if run on a random input,
and the practice of comparing scores among algorithms we call
\emph{expected-reward analysis}.
A higher score indicates a more efficient algorithm.

% D, r
The functions \( D \) and \( r \) are arbitrary and are free to be set in the
context of a particular application.
E.g.\ in graphical user interface software we often desire near-instant
responses, with utility rapidly dropping off with time.
Assuming \(0 \leq r \leq 1\), we immediately see that for all \(A\), partial
or not, we have
\[
  0 \leq S(A) \leq 1 \enspace .
\]
For simplicity in this paper we assume \(r(\omega) = 1\) and \(D\) is an
exponential discount function, i.e.\
$$
    D(c_A(\omega)) = \exp(-\lambda \, c_A(\omega)) \enspace ,
$$
where $\lambda > 0$ is a discount rate.

% P
The choice of \(P\) is also arbitrary;
we remark on two special cases.
If all inputs of a given length are weighted equally, \(P\) is determined
by a probability mass function on \( \mathbb{Z}_{0+} \).
In this case any common discrete probability distribution may be used as
appropriate.
The measure \(P\) is also determined by a probability mass function on
\(\mathbb{Z}_{0+}\) if we weight equal-length inputs according to Solomonoff's
universal distribution \(m\) \cite{lv}, which is a particularly good general
model, although computationally difficult.

% relation to trad (different)
Expected-reward analysis is non-asymptotic, in the sense that all inputs
potentially matter.
Thus, while expected-reward analysis can be used on terminating algorithms, we
expect it to give different results from traditional analysis, in general.
% caching issues
Since particular inputs can make a difference to \(S(A)\), it may be
advantageous to ``hardcode'' initial cases into an algorithm.
This practice certainly exists, e.g. humans may store the \(12 \times 12\)
multiplication table as well as knowing a general integer multiplication
algorithm.

% model invariance
Computational complexity theory often works with classes of problems whose
definitions are equivalent for all ``reasonable'' models of computation
\cite{boas}.
However, even a varying constant factor could arbitrarily change a score.
This is simply the price of concreteness,
and outside of complexity theory, traditional analysis of algorithms generally
selects a particular model of computation and gives precise results that do not
necessarily apply to other models \cite{ac}.

% estimate w/ stats
Unlike traditional analysis, experimental data is relevant to score values
in a statistical sense.
If we are able to generate inputs according to \(P\), either artificially or by
sampling inputs found in practice, \(S(A)\) is a quantity amenable to
statistical estimation.
This suggests a form of experimental analysis of algorithms which focuses on a
single real number rather than plotting the estimated running time for
every input length, which, in the necessary absence of asymptotics in
experimental analysis, may not conclusively rank two competing algorithms
anyway.

% relation to agent scores (similar)
The expected-reward paradigm already appears in the analysis of artificial
agents, rather than algorithms \cite{defi}.
As we see in Sect.~\ref{sec:selfimprovement}, however, even in applications
to AI, working in the more classical domain of algorithms brings benefits.

\subsection{Theory and Practice} \label{sec:tp}

Traditional analysis of algorithms has an established literature going
back decades which provides a set of techniques for performing traditional
analysis on algorithms developed for various problems.
We do not significantly develop a mathematical theory of expected-reward
analysis here, but we make some very brief initial remarks.

By way of introductory example, we consider expected-reward analysis applied to
some well-known sorting algorithms.
Let \(S_n\) be the set of permutations of
\([1..n]\) and let \(\Pi_n\) be a uniform random element of \(S_n\).
We denote the algorithms mergesort and quicksort by
\(M\) and \(Q\), as defined in \cite{iaofa}, and set
\[
  m_n = E\left[\exp(-\lambda \, c_M(\Pi_n)) \right],\ \
  q_n = E\left[\exp(-\lambda \, c_Q(\Pi_n)) \right] \enspace ,
\]
where \(c_A(\omega)\) is the number of comparison operations used by
an algorithm \(A\) to sort an input \(\omega\).

\begin{proposition} \label{pro:miq}
  For \(n \geq 1\) we have
  \begin{equation} \label{eq:mn}
    m_n = \exp \left(-\lambda (n \lceil \lg(n) \rceil + n -
      2^{\lceil \lg(n) \rceil}) \right), \qquad m_0 = 1,
  \end{equation}
  \[
    q_n = \frac{e^{-\lambda(n+1)}}{n} \sum_{k=1}^n q_{k-1} q_{n-k},
      \qquad q_0 = 1 \enspace .
  \]
\end{proposition}
\begin{proof}

From \cite{iaofa},
\(M\) makes the same number of comparisons for all inputs of length $n \geq 1$:
\[
  c_M(\Pi_n) = n \lceil \lg(n) \rceil + n -
  2^{\lceil \lg(n) \rceil} \enspace ,
\]
so (\ref{eq:mn}) is immediate.

Now, when \(Q\) is called on \(\Pi_n\), let $\rho(\Pi_n)$ be the pivot
element, and let \(\underline{\Pi}_n, \overline{\Pi}_n\) be the subarrays
constructed for recursive calls to \(Q\),
where the elements in \(\underline{\Pi}_n\) are less than \(\rho(\Pi_n)\),
and the elements in \(\overline{\Pi}_n\) are greater.

We have
\begin{align*}
  E[ &\exp(-\lambda c_Q(\Pi_n)) ] \\
  &= \frac{1}{n} \sum_{k=1}^n
    E[\exp(\,-\lambda(n + 1 + c_Q(\underline{\Pi}_n) +
    c_Q(\overline{\Pi}_n))\,) \,|\, \rho(\Pi_n)=k] \\
  &= \frac{e^{-\lambda(n+1)}}{n} \sum_{k=1}^n
    E[\exp(\,-\lambda(c_Q(\underline{\Pi}_n) + c_Q(\overline{\Pi}_n))\,)
    \,|\, \rho(\Pi_n)=k] \enspace .
  \end{align*}
It can be seen that given \(\rho(\Pi_n)=k\), \(\underline{\Pi}_n\) and
\(\overline{\Pi}_n\) are independent, thus
\begin{align*}
  E[ &\exp(-\lambda c_Q(\Pi_n)) ] \\
  &= \frac{e^{-\lambda(n+1)}}{n} \sum_{k=1}^n
    E[\exp(-\lambda c_Q(\underline{\Pi}_n)) \,|\, \rho(\Pi_n)=k] \,\cdot \\
    &\hspace{22ex} E[\exp(-\lambda c_Q(\overline{\Pi}_n)) \,|\, \rho(\Pi_n)=k] \\
  &= \frac{e^{-\lambda(n+1)}}{n} \sum_{k=1}^n
    E[\exp(-\lambda c_Q(\Pi_{k-1}))]
    E[\exp(-\lambda c_Q(\Pi_{n-k}))] \enspace . \tag*{$\qed$}
\end{align*}

\end{proof}

From examining the best-case performance of \(Q\), it turns out that \(
c_M(\Pi_n) \leq c_Q(\Pi_n) \) for all \( n \), so the expected-reward
comparison of \(M\) and \(Q\) is easy:
\(S(M) \geq S(Q)\) for any parameters.
However, we may further analyze the absolute scores of \(M\) and \(Q\)
to facilitate comparisons to arbitrary sorting algorithms.
When performing expected-reward analysis on an individual algorithm,
our main desideratum is a way to quickly compute the score value to within a
given precision for each possible parameter value \(P, \lambda\).
Proposition~\ref{pro:miq} gives a way of computing scores of \(M\) and \(Q\)
for measures \(P\) that give equal length inputs equal weight, although it does
not immediately suggest an efficient way in all cases.
Bounds on scores are also potentially useful and may be faster to compute;
in the next proposition, we give bounds on \(m_n\) and \(q_n\) which
are simpler than the exact expressions above.

\begin{proposition} \label{pro:bound}
  For \(n \geq 1\),
  \begin{equation} \label{eq:mnbound}
    \frac{e^{-2 \lambda(n-1)}}{(n-1)!^{\lambda /\log(2)}} \leq m_n
      \leq \frac{e^{-\lambda(n-1)}}{(n-1)!^{\lambda /\log(2)}} \enspace .
  \end{equation}
  For all \(0 < \lambda \leq \log(2) \) and \( n \geq 0 \),
  \[
    \frac{e^{-2 \gamma \lambda (n+1) - \lambda}}
    {(n+1)!^{2\lambda}} (2 \pi (n+1))^\lambda < q_n
        \leq \frac{e^{-2\lambda n}}{(n!)^{\lambda/\log(2)}} \enspace ,
  \]
  where \( \gamma \) is Euler's constant.
\end{proposition}
\begin{proof}
  Sedgewick and Flajolet \cite{iaofa} give an alternative expression for the
  running time of mergesort:
  \[
    c_M(\Pi_n) = \sum_{k=1}^{n-1} \left( \lfloor \lg k \rfloor + 2 \right)
    \enspace .
  \]
  Statement (\ref{eq:mnbound}) follows from this because
  \[
    \log(k)/\log(2) + 1 < \lfloor \lg k \rfloor + 2 \leq \log(k)/\log(2) + 2
    \enspace .
  \]

  With \( 0 < \lambda \leq \log(2) \), we prove the upper bound
  \begin{equation} \label{eq:qn}
     q_n \leq \frac{e^{-2\lambda n}}{(n!)^{\lambda/\log(2)}}
  \end{equation}
  for all \(n \geq 0\) by induction.
  Relation (\ref{eq:qn}) clearly holds for \(n = 0\).
  We show that (\ref{eq:qn}) can be proved for \(n=N\ (N>0)\) on the assumption
  that (\ref{eq:qn}) holds for \( 0 \leq n \leq N-1\).
  Proposition~\ref{pro:miq} gives
  \begin{align*}
    q_N &= \frac{e^{-\lambda(N+1)}}{N} \sum_{k=1}^N q_{k-1} q_{N-k} \\
        &\leq \frac{e^{-\lambda(N+1)}}{N} \sum_{k=1}^N
          \frac{e^{-2\lambda(k-1)}}{((k-1)!)^{\lambda/\log(2)}}
          \frac{e^{-2\lambda(N-k)}}{((N-k)!)^{\lambda/\log(2)}} \\
    \intertext{(by the assumption)}
          &= e^{-3\lambda N + \lambda} \left( \frac{1}{N}
        \sum_{k=1}^N \left(  \frac{1}{(k-1)!}\frac{1}{(N-k)!}
      \right)^{\lambda/\log(2)} \right)\\
    &\leq e^{-3\lambda N + \lambda} \left( \frac{1}{N^{\lambda/\log(2)}}
      \left(  \sum_{k=1}^N \frac{1}{(k-1)!}\frac{1}{(N-k)!}
      \right)^{\lambda/\log(2)} \right) \\
    \intertext{(by Jensen's inequality, since \(0 < \lambda/\log(2) \leq 1\))}
      &= e^{-3\lambda N + \lambda} \left( \frac{(2^{N-1})^{\lambda/\log(2)}}{
      (N!)^{\lambda/\log(2)} } \right) \\
      &= \frac{e^{-2\lambda N}}{(N!)^{\lambda/\log(2)}} \enspace.
    \end{align*}
  Thus (\ref{eq:qn}) has been proved for all \(n \geq 0\).

  For the lower bound on \(q_n\), we use
  the probabilistic form of Jensen's inequality,
  \[
    q_n = E\left[\exp(-\lambda c_Q(\Pi_n)) \right]
    \geq \exp(-\lambda E\left[c_Q(\Pi_n) \right]) \enspace,
  \]
  noting that average-case analysis of quicksort \cite{iaofa} yields
  \[
    E\left[c_Q(\Pi_n) \right] = 2(n+1)(H_{n+1} -1), \qquad n \geq 0 \enspace ,
  \]
  where \( (H_n) \) is the harmonic sequence.
  For \(n \geq 0\), the bound
  \begin{equation*} \label{eq:Hbound}
    H_{n+1} < \log(n+1) + \gamma + \frac{1}{2(n+1)}
  \end{equation*}
  holds \cite{julian2003gamma} (sharper bounds exist), so we have
  \begin{align*}
    q_n &> \exp\left(
          -2\lambda(n+1)\left(\log(n+1) + \gamma + \frac{1}{2(n+1)} -1\right)
          \right) \\
          &= e^{-2(\gamma-1)\lambda (n+1) -\lambda}
            (n+1)^{-2\lambda(n+1)} \enspace .
  \end{align*}
  We finish by applying Stirling's inequality
  \[
    (n+1)^{-(n+1)} \geq \sqrt{2\pi (n+1)} e^{-(n+1)}/(n+1)!,
    \qquad n \geq 0 \enspace . \eqno\qed
  \]
\end{proof}

%\begin{rem}
%  These facts are easily verified for Proposition \ref{pro:bound}:
%  \begin{enumerate}
%    \item For \(n \geq 2 \) the upper bound on \( m_n \) is strict.
%    \item For \(n \geq 3 \) the lower bound on \(m_n\) is strict.
%
%    \item For \(n \geq 3\) the upper bound on \(q_n\) is strict iff
%      \(\lambda < \log(2)\).
%    \item The lower bound as stated for \(q_n\) is strict, however, sharper
%      bounds exist (e.g.\ (\ref{eq:Hbound}) may be tightened).
%  \end{enumerate}
%\end{rem}

From these results we may get a sense of the tasks involved in
expected-reward analysis for typical algorithms.
We note that with an exponential discount function, the independence of
subproblems in quicksort is required for obtaining a recursive formula,
whereas in traditional average-case analysis, linearity of expectation
suffices.

We end this section by mentioning an open question relevant to a theory of
expected-reward analysis.

\begin{question} \label{q:sup}
  If we fix a computational problem and parameters \(P, \lambda\),
  what is \( \sup_A S(A) \), and is it attained?
\end{question}

If \( \sup_A S(A) \) is not attained then the situation is similar to that in
Blum's speedup theorem.
Comparing \(\sup_A S(A)\) among problems would be the expected-reward
analog of computational complexity theory but because of the sensitivity
of \(S\) to parameters and the model of computation, this is not useful.

\section{Self-Improving AI} \label{sec:selfimprovement}

The generality of \(\zp\) problems allows us to view design and analysis
of \ZPAs\ as a task which itself may be given to a \ZPA, bringing about
recursive self-improvement.
Here we present one possible concrete example of this notion and discuss
connections with AI.

Computational problems with Turing degree \(\zp\) are Turing-equivalent so
without loss of generality in this section we assume \ZPAs\ are
automated theorem provers.
Specifically, we fix a formal logic system, say ZFC (assuming it is consistent),
and take the set of inputs to be ZFC sentences, and the possible outputs to be
\texttt{provable} and \texttt{not provable}.

Let a predicate \(\beta\) be such that \(\beta(Z)\) holds
iff \(Z\) is a \ZPA\ which is correct on provable inputs and does not terminate
otherwise.
In pseudocode we write the instruction to run some \(Z\) on input \(\omega\) as
\(Z(\omega)\), and if \(\omega\) contains \(\beta\) or \(S\) (the score function),
their definitions are implicitly included.

We give an auxiliary procedure \textsc{Search} which takes as input
a \ZPA\ \(Z\) and a rational number \(x\) and uses \(Z\) to obtain a \ZPA\ which
satisfies \(\beta\) and has score greater than \(x\) (if possible).
Symbols in bold within a string literal get replaced by the value of the
corresponding variable.
We assume \ZPAs\ are encoded as strings in a binary prefix code.

\vspace{0.5cm}
\begin{algorithmic}[1]
  \Procedure{Search}{$x, Z$}
    \State \(u \gets \text{the empty string}\)
    \Loop
    \StartDo in parallel until one returns \texttt{provable}:
        \State A: \(Z(\text{``}\exists v:
          (Z^* = \mathbf{u}0v \implies
          \beta(Z^*) \land S(Z^*)> \mathbf{x})\text{''})\)
        \State B: \(Z(\text{``}\exists v:
          (Z^* = \mathbf{u}1v \implies
          \beta(Z^*) \land S(Z^*)>\mathbf{x} )\text{''})\)
        \State C: \(Z(\text{``} Z^*=\mathbf{u} \implies
          \beta(Z^*) \land S(Z^*)> \mathbf{x}\text{''})\)
      \EndDo
      \If{A returned \texttt{provable}}
        \State \( u\gets u0 \)
      \EndIf
      \If{B returned \texttt{provable}}
        \State \(u \gets u1 \)
      \EndIf
      \If{C returned \texttt{provable}}
        \State return \( u \)
      \EndIf
    \EndLoop
  \EndProcedure
\end{algorithmic}
\vspace{0.5cm}

We remark that the mechanism of \textsc{Search} is purely syntactic and does
not rely on consistency or completeness of ZFC, or the provability thereof.
This would not be the case if we strengthened \(\beta\) to require that
\(\beta(Z)\) is true only if at most one of \(Z(\omega)\) and \(Z(\neg
\omega)\) returns \texttt{provable}.
Such a \(\beta\) would never provably hold in ZFC.

The following procedure \textsc{Improve}
takes an initial \ZPA\ \(Z_0\) and uses dovetailed calls to \textsc{Search} to
output a sequence of \ZPAs\ that tend toward optimality.
\vspace{0.5cm}
\begin{algorithmic}[1]
  \Procedure{Improve}{$Z_0$}
  \State \(best\) \( \gets Z_0 \),\ \ 
    \( pool \gets \{ \} \),\ \ 
    \( score \gets 0 \)
  \For{\( n \gets 1 \textbf{ to } \infty \)}
    \State \(a_n \gets n\)th term in Stern-Brocot enumeration of \(\mathbb{Q}
      \cap (0,1] \)
    \If{\( a_n > score \)}
      \State \(initialState \gets\) initial state of
        \textsc{Search}\((a_n, best)\)
      \State add \( (a_n, best, initialState) \) to \(pool\)
    \EndIf
    \State \(improvementFound \gets \mbox{false}\)
    \For{\((a, Z, state)\) in \(pool\)}
      \State run \textsc{Search}(\(a, Z\)) one step starting in state \(state\)
      \State \(newState \gets\) new current state of \textsc{Search}(\(a, Z\))
      \If{\(state\) is not a terminating state}
      \State in \(pool\), mutate \((a,Z,state)\) into \((a,Z,newState)\)
        \State continue
      \EndIf
      \State \(improvementFound \gets \mbox{true}\)
      \State \(best \gets \) output of \textsc{Search}(\(a, Z\))
      \State \(score \gets a\)
      \For{ \((\hat{a},\hat{Z},\hat{state})\) in \(pool\)
        where \( \hat{a} \leq score \)}
        \State remove \((\hat{a},\hat{Z},\hat{state})\) from \(pool\)
      \EndFor
      \State print \(best\)
    \EndFor
    \If{\(improvementFound\)}
      \For{ \((a,Z,state)\) in \(pool\) }
        \State \(initialState \gets\) initial state of
          \textsc{Search}\((a, best)\)
        \State add \((a, best, initialState) \) to \(pool\)
      \EndFor
    \EndIf
  \EndFor
  \EndProcedure
\end{algorithmic}
\vspace{0.5cm}

The procedure \textsc{Improve} has the following basic property.

\begin{proposition}
  Let \((Z_n)\) be the sequence of \ZPAs\ printed by \textsc{Improve}.
  If \(\beta(Z_0)\) holds, and if there is any \ZPA\ \(Y\) and
  \(s \in \mathbb{Q} \) where \(\beta(Y)\) and \(S(Y) > s > 0\) are provable,
  we have
\[
  \lim_{n \rightarrow \infty} S(Z_n) \geq s \enspace .
\]
If \( (Z_n) \) is finite, the above limit can be replaced with the
last term in \( (Z_n) \).
\end{proposition}
\begin{proof}
  The value \( s \) appears as some value \(a_n\).
  For \(a_n = s \), if \(a_n > score\) in line 5, then
  \textsc{Search}\((s, best)\) will be run one step
  for each greater or equal value of \(n\) and either terminates (since
  \(Y\) exists) and
  \(score\) is set to \(s\), or is interrupted if we eventually
  have \(score \geq s\) before \textsc{Search}\((s, best)\) terminates.
  It suffices to note that when \(score\) attains any value \(x > 0\),
  all further outputs \(Z\) satisfy \(S(Z) > x\) and there
  is at least one such output. \qed
\end{proof}
The procedure \textsc{Improve} also makes an attempt to use
recently printed \ZPAs\ in calls to \textsc{Search}.
However, it is not true in general that \( S(Z_{n+1}) \geq S(Z_n) \).
Checking if a particular output \(Z_n\) is actually an improvement over \(Z_0\)
or \(Z_{n-1}\)
requires extra work.

In artificial general intelligence (AGI) it is desirable to have intelligent
systems with the ability to make autonomous improvements to themselves
\cite{gm}.
If an AGI system such as an AIXI approximation \cite{aixi}
already uses a \ZPA\ \(Z\) to compute the universal distribution \(m\), we can
give the system the ability to improve \(Z\) over time by devoting some of its
computational resources to running \textsc{Improve}.
This yields a general agent whose environment prediction ability tends toward
optimality.

\section{Future Work} \label{sec:furtherwork}

We would like to be able to practically use expected-reward
analysis with various parameter values, probability measures, and discount
functions, on both terminating and non-terminating algorithms.
Particularly, we would like to know whether \ZPAs\ may be practically
analyzed.
It may be possible to develop general mathematical tools and techniques to
enhance the practicality of these methods, such as exist for traditional
analysis;
this is a broad and open-ended research goal. \\

\textbf{Acknowledgements.}
The author wishes to thank
    Zhicheng Gao,
    Nima Hoda,
    Patrick LaVictoire,
    Saran Neti, and
    anonymous referees
    for helpful comments.

\bibliographystyle{splncs}
\bibliography{AoAaPA}
\end{document}